\documentclass{ecai}

\usepackage{times}
\usepackage{graphicx}
\usepackage{latexsym}
\usepackage{mathtools}
\usepackage{changepage}
\usepackage{cite}
\usepackage{xcolor}
\usepackage{amsmath}
\newcommand{\lkakko}{\ensuremath{\{ }}
\newcommand{\rkakko}{\ensuremath{\}}}
\newcommand{\Maxi}{\textsf{Max}}

\newcommand{\Pref}{\textsf{Pref}}

\newcommand\myarrow{\stackrel{\mathclap{{\tiny \normalfont\mbox{alert}}}}{\leftarrow}}
\usepackage{amssymb}
\usepackage{enumitem}
\newcommand{\Rmax}{\ensuremath{V^{\text{max}}}}

\newcommand{\hide}[1]{}

\newcommand{\staritem}{\global\asterisktrue\item}
\newcommand{\perhapsstar}{
    \ifasterisk$\star$\global\asteriskfalse\fi
}
\newif\ifasterisk

\newcommand{\view}{\ensuremath{\textsf{View}_R}}
\newcommand{\Delete}{\textsf{Del}}

\newcommand{\Attacker}{\textsf{Attacker}}

\newcommand{\cons}[2]{\ensuremath{#1 \leftarrow
        #2}}
\newcommand{\badcons}[2]{\ensuremath{#1 \nleftarrow
        #2}}
\newcommand{\alertcons}[2]{\ensuremath{#1 \myarrow
        #2}}

\newcommand{\Sentence}{\textsf{S}}
\newcommand{\CaSpare}{\textsf{Kspare}}
\newcommand{\Dob}{\textsf{Dob}}
\newcommand{\Car}{\textsf{Car}}
\newcommand{\Bb}{\textsf{Bb}}
\newcommand{\Tz}{\textsf{Tz}}
\newcommand{\Ca}{\textsf{Ca}}
\newcommand{\Dig}{\textsf{Dig}}
\newcommand{\ARB}{\textsf{ARB}}
\newcommand{\ACE}{\textsf{ACE}}
\newcommand{\Dr}{\textsf{Dr}}
\newcommand{\Dih}{\textsf{Dih}}
\newcommand{\Md}{\textsf{Md}}
\newcommand{\Hr}{\textsf{Hr}}
\newcommand{\La}{\textsf{La}}
\newcommand{\Nif}{\textsf{Nif}}
\newcommand{\hpot}{\textsf{hpot}}
\newcommand{\hpom}{\textsf{hpom}}
\newcommand{\hpl}{\textsf{hpl}}
\newcommand{\igt}{\textsf{igt}}
\newcommand{\hf}{\textsf{hf}}
\newcommand{\hpt}{\textsf{ht}}
\newcommand{\asth}{\textsf{asth}}
\newcommand{\bc}{\textsf{bc}}
\newcommand{\lve}{\textsf{lve}}
\newcommand{\tac}{\textsf{tac}}
\newcommand{\pad}{\textsf{pad}}
\newcommand{\anp}{\textsf{anp}}
\newcommand{\gt}{\textsf{gt}}
\newcommand{\asp}{\textsf{asp}}
\newcommand{\postMi}{\textsf{post-mi}}
\newcommand{\ckdplus}{\textsf{ckd}+}
\newcommand{\ckdminus}{\textsf{ckd}-}

\newcommand{\dm}{\textsf{dm}}
\newcommand{\ms}{\textsf{ms}}
\newcommand{\os}{\textsf{os}}
\newcommand{\hprk}{\textsf{hprk}}
\newcommand{\hpok}{\textsf{hpok}}
\newcommand{\pr}{\textsf{pr}}
\newcommand{\ane}{\textsf{ae}}
\newcommand{\ras}{\textsf{ras}}
\newcommand{\old}{\textsf{old}}
\newcommand{\arr}{\textsf{arr}}
\newcommand{\female}{\textsf{female}}
\newcommand{\severe}{\textsf{severe}}
\newcommand{\aph}{\textsf{apheresis}}
\newcommand{\dia}{\textsf{dia}}

\newcommand{\orC}{\textsf{or}}
\newcommand{\andC}{\textsf{and}}

\newtheorem{example}{Example}

\newtheorem{definition}{Definition}
\newtheorem{proof}{\normalfont {\it Proof.}}
\newtheorem{proposition}{Proposition}
\newtheorem{lemma}{Lemma}

\begin{document}

\title{Coalition Formability Semantics with Conflict-Eliminable Sets of Arguments
}


\author{Ryuta Arisaka\institute{National Institute of Informatics, Japan, email: ryutaarisaka@gmail.com} \and Ken Satoh \institute{National Institute of Informatics,
Japan, email: ksatoh@nii.ac.jp} }

\maketitle
\bibliographystyle{ecai}

\begin{abstract}
    We consider abstract-argumentation-theoretic
    coalition formability
    in this work. Taking a model in
    political alliance among political parties,
    we
    will contemplate profitability, and then formability,
    of a coalition. As is commonly
    understood, a group forms
    a coalition with another group
    for a greater good, the goodness measured against
    some criteria.
    As is also commonly understood, however, a coalition may deliver
    benefits to a group X at the sacrifice of something
    that X was able to do before coalition formation,
    which X may be no longer able to do under the coalition.
    Use of the typical
    conflict-free sets of arguments
    is not very fitting for
    this aspect of coalition,
    which prompts
    us to turn to a weaker notion,
    conflict-eliminability, as a
    property that a set of arguments should
    primarily satisfy. We require numerical
    quantification of attack strengths
    as well as of argument strengths
    for its characterisation.
    We will first analyse semantics of profitability
    of a given conflict-eliminable set forming
    a coalition with another conflict-eliminable
    set, and will then provide
    four coalition formability semantics,
    each of which  formalises
    certain utility postulate(s) taking
    the coalition profitability into account.
\end{abstract}
\section{Introduction}
Coalition formation among agents is an important topic
in many domains including economics, political science,
and computer science. Two groups of agents,
by teaming up together, could
achieve a task which they cannot otherwise
do on their own. Exploring abstract argumentation
theory
for finding
an apt characterisation of coalition
formability looks specially rewarding since
it is reasonable that we regard a coalition
as a set of arguments its members express.
There are already a few
papers in the literature looking
at this subject matter:
with preference-based argumentation frameworks
and task allocations
\cite{Amgoud05}; with a cooperative
goal generation and fulfilling \cite{Boella08},
respecting the property of reciprocity,
i.e. agents give to the coalition they are in
and benefit from it; and with dialogue games and
pay-offs \cite{Riley12}. Shared by them
is the theme of identifying a group
of individual agents who optimise benefits (or social welfare)
to themselves by being in the group.
The optimal group thus formed is free of
internal conflicts, and the participating agents do
not have to give up anything by being a part.
We consider this
kind of a coalition supportive. \\
\indent The sort of coalition formation we have
in mind, on the other hand, is one that may be
found in  political
alliance among political parties. Such alliance
is motivated if, for example, political parties
want to reach the voting threshold of passing certain bills
or want to win national elections.
A political-alliance-like coalition exhibits the following
unique characteristics:
\begin{description}[leftmargin=0.2cm]
    \item[\textbf{More organisational than individual}]
        It is not possible to freely move agents
        across multiple political parties such as to
        connect
        those having close interests together
        for optimal political party formation.
        A participant to a political party
        is often expected to stay in the party during
        the parliament term.
        The assumption of self-interested agents
        as studied in \cite{Riley12}
        does not fit very well here. There are
        repercussions to profitability of a coalition,
        too, in that it is primarily for
        a party's, or parties', benefits than
        the participants' benefits
        that a political alliance is formed.
    \item[\textbf{Partial internal conflicts}]
        Agents in a political party should support the party's
        agendas and policies. Hence they do not defeat each other about them, broadly spoken.
        It is common, however,
        that there are smaller factions within
        a political party arguing against one another
        over details.
        As a consequence, some participating agents
        may be unable to proclaim their opinions on certain
        policies as the party's opinions. Put another
        way, some individual opinions may be suppressed
        for the party.
    \item[\textbf{Asymmetry in attacks to and from a coalition}]
         For a political alliance to retain any credibility
         of the arguments it expresses,
         it must argue
         only by the conflict-free, i.e. self-contradiction-free, portion
         of the arguments of the party's participants. But the other political
         parties not in the alliance are unhindered
         by the personal circumstance
         of the alliance. If a political party
         A is in alliance with another political party
         B, then an external party C can argue
         against any argument of the individual
         participants in A or in B in order
         to criticise not just the individuals
         but the coalition.
     \item[\textbf{Better larger than smaller}]
         In \cite{Amgoud05},
         the rate of defections is associated to
         the number of agents in a coalition, thus
         a smaller set preferred. That does not
         carry over here: if one single political
         party or one single political alliance dominates the parliament,
         it has total freedom in policy making, which
         is clearly desirable. While not
         primarily on abstract argumentation, there is
         a work \cite{Bulling08} on alternating-temporal logic
         incorporating the framework of \cite{Amgoud05}.
         In the logic,
         a larger set is better.
\end{description}
\begin{example}
    \normalfont
    The Liberal Democratic Party of Japan (LDP) is traditionally
     the most influential but also a complex party
     of multiple factions. Its goals are:
     rapid, export-based economic
     growth; close cooperation with the U.S. in foreign
     and defence policies; simplification and streaming
     of government bureaucracy; privatization of
     state-owned enterprises; and adoption of measures
     such as tax reform for the ageing society.
     There are three major factions in the LDP:
     \begin{description}[leftmargin=0.2cm]
         \item[\normalfont Heisei Kenkyukai(A)'s] promises
             include
              international cooperation with China and Korea,
              construction of highways, a Gasoline tax, and protection of
              small farmers and discriminated peoples.
          \item[\normalfont Kouchi Kai(B)] promises
               international
              cooperation with China and Korea, a government
              bond and consumption tax for
              national medical care and national banks
              which financially support small firms, and
              free trade policy.
          \item[\normalfont Seiwa Seisaku Kenkyukai(C)] promises
               tax reduction for high income taxpayers
               and large companies,
               a strong military relationship
               with the U.S. for national defence issues,
               visits to Yasukuni Shrine,
               reduction of road and railway construction,
               free trade for car exports,
               lesser medical care and removal of
               protection of small farmers.
     \end{description}
    Infighting is evident. The argument of A's and B's
    for a close cooperation with China and Korea is
    dampened by C's visits to Yasukuni Shrine,
    A's argument for protecting small farmers
    and construction of highways
    is in direct conflict with  promises of C's
    to remove that protection and to reduce public construction, and
    B's free trade policy is not wholeheartedly
    welcomed by C which promises free trade chiefly for car exports.
    With all these discordance, they are
    still united in the LDP because the number matters
    in  politics: losing out in national elections and in bill-passing
    as the result of forming an independent - and
    less dominant - political party
    is the greater harm to them.
    Now, the LDP as a whole
    does not (and obviously cannot) express the visit to
    Yasukuni Shrine as the LDP's policy due to
    the internal conflict.
    However, such a circumstance means little
    to other political
    parties. To give evidence, the Japanese Communist
    Party has long criticised
    the Yasukuni Shrine visit by some individual
    members of the LDP as a way of criticising
    the LDP itself. There indeed is asymmetry in
    attacks.
\end{example}
In this work, we contemplate
abstract-argumentation-theoretic
characterisations of profitability,
and then formability, of a coalition of this kind.
To be more specific, we consider the following two questions:
(1) suppose a set of arguments that may contain
partial internal conflicts (as an abstract representation
of a political party) and suppose also rational
criteria
of coalition profitability,
with which other (disjoint)
sets of arguments that may also contain partial internal
conflicts (i.e. representations of other political parties)
can it profit from forming a coalition?; and (2)
suppose the profitability relation, suppose some rational principles to judge the goodness
of a coalition, and suppose such a set of arguments,
with which other (disjoint) such sets of arguments
can it actually form a coalition? \\
\indent The following are particularly interesting technicalities
of these semantics. First and foremost,
the above-described coalition formability is not
simply about whether
the resulting coalition is acceptable, which
could be handled by adapting the standard acceptability semantics
in abstract argumentation theory, but about whether
a set of arguments potentially having partial
internal conflicts can form a coalition with
another similar set. The former concerns
the state of the resulting set, while
the latter must be parametrised
by coalition profitabilities of both sets.
Secondly, because of the presence of
partial internal conflicts and of
the asymmetry in attacks
to and from a coalition, we have to:
(1) accommodate a weaker notion than
conflict-freeness as a property that
a set of arguments should primarily satisfy -
{\it conflict-eliminability} as we term it,
which permits members of a set to attack
other members of the same set so long as
none of them is completely defeated; (2)
obtain {\it intrinsic arguments}
of a conflict-eliminable set, which are
the arguments that would remain if
any partial internal conflicts within
the set were resolved away\footnote{This should
    not be confused. We are not meaning
    that some arguments
    would disappear as the result.
    The internal conflicts are assumed
    non-defeating. We rather mean that some arguments
    may be weakened of their potency. See \textbf{Partial
        internal conflicts}. What would remain are
then those arguments unaffected by the partial internal conflicts
and those weakened arguments.}
    ; and (3) use
the intrinsic arguments to determine
which external arguments are being attacked
by the conflict-eliminable set,
while still keeping the original arguments
of the set in order to determine if
it is being attacked by external arguments (see \textbf{
    Asymmetry in attacks to and from a coalition} above).
Here again, our task is not just
whether some set of arguments satisfies
conflict-eliminability: we must consider
if any attacks are strong enough to defeat
an argument, and, in case an argument is attacked
but not defeated, how much it would be
weakened/compromised by the attacks.
To cope with these, we attach {\it argument capacity},
a numerical value, to each argument,
and an attack strength, again a numerical value,
to each attack.
The idea
is: a set of arguments is conflict-eliminable
just when none of the members of the set attack an argument
of the same set with a greater numerical value
than the argument's capacity. As the argument capacities
and attack strengths are both numerical,
it is easy to derive its intrinsic arguments and
their attacks on external arguments.
\subsection{Related work}
We are not aware of other works in
the literature of abstract argumentation theory
dealing with this kind of political-alliance-like coalition
formation, although an earlier draft of this work
has been already cited in \cite{Arisaka16b}
for a more specific, logic-oriented
instantiation with classical logic sentences
and belief contraction. Also, to the best of
our knowledge, the previous
approaches proposed in the abstract argumentation literature
are not self-sufficient for dealing with
the two above-mentioned technicalities.
We have already mentioned the key works
on coalition formation \cite{Amgoud05,Boella08,Riley12}.
They apply Dung's acceptability semantics
\cite{Dung95} for
characterising acceptability of a coalition
that is individual-benefit-oriented, that is conflict-free,
and that generally prefers a smaller set. In Section 4
of \cite{Amgoud05} and in \cite{Boella08}, a coalition
is associated with a set of tasks/goals, and conflicts
between coalitions are measured such as by competition
which occurs when two coalitions share the same
tasks/goals.
We, however,
focus on the described political-alliance-like
coalition profitability/formability semantics
with conflict-eliminable sets of arguments.
We do not consider the meta-knowledge of tasks or goals.
Instead, we measure profitability of a coalition for
a conflict-eliminable set of arguments by three criteria:
(1) the size of the coalition; (2) whether
the number of attackers to the conflict-eliminable set of arguments increases
or decreases in the coalition;
and (3) how defended the coalition is from external
arguments. Coalition formability is relativised
to profitabilities of two conflict-eliminable sets.
A rather different perspective of coalition formation:
calculation of probabilistic likelihood of a coalition formation
capable of achieving some task, given the prior probability
of agents' joining in a coalition and of preventing
other members from joining in the coalition,
was highlighted in probabilistic argumentation frameworks
\cite{Li12}. Such quantitative judgement is out of the scope
of this work. \\
\indent We mention other works relevant to ours.
\begin{description}[leftmargin=0.2cm]
    \item[\textbf{Attack-tolerant abstract argumentation}]
         Characterisation of acceptability
         semantics for a non-conflict-free
         set of arguments is gaining
         attention.
         The 2-valued semantics \cite{Pereira07}
         makes use of reductio ad absurdum
         to resolve inconsistency. Conflict-tolerant semantics \cite{Arieli12}
relaxes conflict-freeness by using four values
(accept, reject, no opinion, and mixed feeling)
to label arguments for paraconsistent
abstract argumentation. However, these approaches
do not incorporate numerical values,
which makes it difficult to reason about the strength
of attacks. Weighted argument systems \cite{Dunne11}
attach numerical values to attack relations.
There is also a system-wide numerical value called
the inconsistency budget. In their systems,
conflicts in a set of arguments are quantified
as the sum of numerical values given to the attack relations
appearing in the set. If the sum
does not exceed the given inconsistency budget, then
the set is considered para-conflict-free.
Their acceptability semantics is relative to
the global inconsistency budget. In our setting,
having numerical attack strengths alone is
not sufficient, as we must know intrinsic arguments
of a conflict-eliminable set, which
relies
both on attack strengths and on
argument strengths. We do not
use any global
and uniform budget. Further, substantively we do not tolerate
any inconsistency: intrinsic arguments must be conflict-free.
Social abstract argumentation frameworks \cite{Leite11},
to which an equational approach  \cite{Gabbay14} also relates,
attach numerical values to arguments in the form
of {\it for} votes and {\it against} votes. They
allow for fine-grained para-consistency.
We could potentially  adapt their approach
to characterise
our conflict-eliminability. However,
their numerical attack characterisations by votes are quite specific.
We choose more abstract, axiomatic characterisations.
In the literature, axiomatic approaches have been considered
to, for instance, ensure logical consistency
of an argumentation framework \cite{Amgoud16,Amgoud09,Caminada05}.
Classifications of attack relations by axioms
they satisfy have been also done \cite{Gorogiannis11}. The axiomatic
approaches help regulate
an abstract argumentation system from a general
standpoint.
   \item[\textbf{Dynamic abstract argumentation}]
       Our framework possesses
       a certain kind of dynamic nature.
       So far dynamic changes have been considered
       within the literature of abstract argumentation
       to: assume structural argumentation \cite{Dung09,Bex03,Prakken04,Caminada07,Prakken10}
       and modify non-falsifiable facts \cite{Rotstein08};
       add a new argument \cite{Cayrol10,Oikarinen11,Baumann10b};
       revise
       attack relations \cite{Coste-Marquis14};
       revise an argumentation framework by encoding
       it into propositional logic \cite{Coste-Marquis14b}; and revise
       an argumentation framework with an argumentation
       framework, see \cite{Baumann15}.\footnote{There are
       also analysis on mutability of acceptable
       arguments by adding or removing an argument
       and/or an attack \cite{Boella09a,Boella09b}.}
       Given an argumentation
       framework, these works calculate a revised
       argumentation framework. That is, they
       derive a post-state from a pre-state
       given some input.  We, however, require interactions
       between the initial set of arguments (the pre-state)
       and intrinsic arguments of a conflict-eliminable
       set (post-states) due to the asymmetry in attacks
       to and from a coalition.
       The pre-state/post-state coordinations
       are, as far as we are able to fathom, not dealt with in the
       above-mentioned
       studies. Meanwhile, one of the works
       on coalition formation as mentioned earlier, namely
       \cite{Boella08}, admits coalitional and
       non-coalitional views of agents. From
       a non-coalitional view of agents,
       a number of coalitional views may be derived.
       Still, what they consider are conflicts among
       coalitional views (for goal fulfillment), which is
       a problem possessing a different nature.
\end{description}
In the rest, we will:
recall Nielsen-Parsons' argumentation frameworks \cite{Nielsen06}
that generalise Dung's ones with group attacks (Section 2);
introduce our argumentation frameworks for conflict-eliminable
sets of arguments and find a link to Nielsen-Parsons'
frameworks as a side contribution (Section 3); and develop
semantics for profitability, and then formability,
of coalition formation, at the same time presenting  theoretical
results (Section 4), before
drawing
conclusions.
\section{Preliminaries}
While Dung's argumentation frameworks \cite{Dung95}
are the most important in the abstract argumentation literature,
Nielsen-Parsons' generalised versions
with
group attacks are probably closer to our
own. Let us recall the key definitions
of their frameworks.\\
\indent An {\it  argument} is an abstract entity,
and the class of all arguments is $\mathcal{A}$.
An {\it argumentation framework}
is a tuple $(A, G)$ where
$A \subseteq_{\text{fin}} \mathcal{A}$
and $G: (2^{\mathcal{A}} \backslash \emptyset) \times
\mathcal{A}$. A set $A_1 \subseteq
A$ is said to {\it attack}
an argument $a \in A$ if and only if,
or simply
iff,
$(A', a) \in G$ for some $A' \subseteq A_1$.
We say that $(A', a) \in G$ is {\it minimal}
iff for every $A'' \subset A'$, $(A'', a) \not\in G$.
A set $A_1 \subseteq A$
is {\it conflict-free} iff there exists
no $a \in A_1$ such that $A_1$ attacks $a$.
A set $A_1 \subseteq A$
{\it defends} $a \in A$ iff
if $(A_x, a) \in G$ for $A_x \subseteq A$
is minimal, then
$A_1$ attacks some $a_x \in A_x$.
A set $A_1 \subseteq A$
accepts $a \in A$ iff
$A_1$ defends $a$.
A set  $A_1 \subseteq A$
is admissible iff $A_1$ accepts all its members.
A set $A_1 \subseteq A$
is a preferred set (extension) iff
$A_1$ is admissible and
there exists no $A_1 \subset A_2 \subseteq A$
such that $A_2$ is admissible.
There are other notions such as
complete sets (extensions),
stable sets (extensions), and
the grounded set (extension). An interested
reader will find more information in \cite{Dung95,Nielsen06}.
\section{Argumentation Frameworks
    for Conflict-Eliminable Sets of Arguments}
Let $\mathbb{N}$ be the class of natural numbers
including 0,
and let $\mathcal{S}$ be $\mathcal{A} \times
\mathbb{N}$.
We refer to any element of
$\mathcal{S}$ by $s$ with or without a subscript.
In our development, members of
$\mathcal{S}$ (not of $\mathcal{A}$) are
arguments.
Not just any set of arguments
will we be interested in, however.
\begin{definition}[Coherent sets of arguments]
    \normalfont
    Let $S_1$ be a subset of $\mathcal{S}$.
    We say that $S_1$ is coherent
    iff $S_1$ satisfies the following conditions.
    \begin{enumerate}[leftmargin=0.5cm]
        \item $S_1$
    is a finite subset of $\mathcal{S}$.
\item For any $(a, n) \in S_1$,
    it holds that $n > 0$.
\item For any $(a, n) \in S_1$,
    there is no $m \not= n$ such that
    $(a, m) \in S_1$.
\end{enumerate}
\end{definition}
An argumentation framework usually
satisfies the first condition.
To explain the second condition, we mention
that each argument has {\it argument capacity}.
For $(a, n) \in \mathcal{S}$, $a$ is its {\it identifier}, and
$n$ is its {\it capacity}. The greater
the capacity of an argument is, the more information
it contains. That an argument has
argument capacity of 0 basically means
that it has no utility. The third property is to ensure
that each argument identity is used
by at most one argument in a chosen
subset of $\mathcal{S}$.
From here on, by $S$ with or without a subscript
we denote a coherent set of arguments.
\begin{example}
    \normalfont
    Consider four arguments:
\begin{enumerate}
    \item $a_1$: I support cooperation with China and Korea.
         I support construction of highways. I support
         protection of small farmers. I support a Gasoline tax.
     \item $a_2$: I support cooperation with China and Korea.
         I support free trade. I support consumption tax
         for national medical care and national banks.
     \item $a_3$: I support visits to Yasukuni Shrine. I am against
          construction of highways. I support
          a close military relationship with the U.S.
          I support free trade,
          but only for car exports. I am against
          protection of small farmers.
      \item $a_4$: I denounce visits to Yasukuni Shrine.
\end{enumerate}
    We could (though do not have to)
    treat the number of sub-arguments of each argument as its
    argument capacity so that we have
    $(a_1, 4)$, $(a_2, 3)$, $(a_3, 5)$, $(a_4, 1)$.
\end{example}
\indent Now, let $R$ be
a partial function $2^{\mathcal{S}}  \times \mathcal{S} \rightharpoonup \mathbb{N}$ that
satisfies the following conditions (or axioms). Informally,
$R(S, s)$ represents the attack strength
of $S$'s attack on $s$. In the below, `$R$
is defined for $(S, s)$' is synonymous
to `$R(S, s)$ is defined'.
\begin{enumerate}
    \item $R$ is undefined for
        $(\emptyset, s)$ for any
        $s \in \mathcal{S}$ \textbf{[Coherence]}.
    \item For any $S_1 \subseteq S \subseteq_{\text{fin}} \mathcal{S}$
        and for any $s \in S$,
        if $R$ is defined for
        $(S_1, s)$, then
        $R$ is defined for
        any $(S_2, s)$ for $\emptyset \subset S_2 \subseteq S_1$.
        \textbf{[Quasi-closure by
            subset relation]}.
    \item For any $S_1, S_2 \subseteq S \subseteq_{\text{fin}}
        \mathcal{S}$ and for any $s \in S$,
        if $R$ is defined both for $(S_1, s)$
        and for $(S_2, s)$, then
        $R$ is defined also for $(S_1 \cup S_2, s)$
        \textbf{[Closure by set union]}
    \item For any $S_1 \subseteq S \subseteq_{\text{fin}}
        \mathcal{S}$ and for any $s \in S$ such that
        $R$ is defined
        for $(S_1, s)$, it holds that
        $R(S_1, s) > 0$ \textbf{[Attack with a positive strength]}.
    \item For any $(a, n), (a, m) \in \mathcal{S}$
        such that $n \leq m$, the following
        holds true:
        if $R(S_1, s)$
        for
         some $s \in \mathcal{S}$ and
        for some $S_1 \subseteq_{\text{fin}} \mathcal{S}$
         such that $(a, n) \in S_1$ is defined, then
         $R(S_2, s)$ for
        $S_2 = (S_1 \backslash (a, n)) \cup (a, m)$
        is defined and is such that
        $R(S_1, s) \leq R(S_2, s)$
        \textbf{[Attack monotonicity 1 (source)]}.
        \item For any $S_1, S_2 \subseteq S \subseteq_{\text{fin}} \mathcal{S}$
        and for any $s \in S$, if $R$ is defined
        for $(S_1, s)$, $(S_2, s)$ and
        $(S_1 \cap S_2, s)$,
        then $R(S_1 \cap S_2, s) \leq
        R(S_i, s)$ for both $i = 1$ and $i = 2$
        \textbf{[Attack monotonicity 2 (source)]}.
\item For any $(a, n), (a, m) \in \mathcal{S}$
        such that $n \leq m$, it holds that if
        $R$ is defined for $(S_1, (a, n))$ for some
        $S_1 \subseteq_{\text{fin}} \mathcal{S}$
        such that $S_1 \cap
        \bigcup_{l \in \mathbb{N}}\{(a, l)\} = \emptyset$,
        then it is defined for
        $(S_1, (a, m))$, and, moreover,
        $R(S_1, (a, n)) \leq R(S_1, (a, m))$
        \textbf{[Attack monotonicity 3 (target)]}
    \item For any $S_1 \subseteq_{\text{fin}} \mathcal{S}$
        and for
        any $s \in
        \mathcal{S}$,
        $R$ is undefined for $(S_1, s)$ if
        $s \in S_1$.
        \textbf{[No self attacks]}.
\end{enumerate}
\textbf{[Coherence]} ensures that an attack must come
from some argument(s). \textbf{[Quasi-closure by
    subset relation]} ensures that there
is a group attack from a set of arguments
on an argument
just because each member of the set is attacking
the argument. This can be contrasted
with the group attacks in Nielsen-Parsons' argumentation frameworks. But it must
be noted, as we are to mention shortly, that
that there is an attack of an argument
$(a_1, n_1)$
on another argument $(a_2, n_2)$ does not
mean that $(a_1, n_1)$ defeats $(a_2, n_2)$
in our framework.
\textbf{[Closure by set union]} is the reverse
of \textbf{[Quasi-closure by subset relation]}. The purpose
of \textbf{[Attack with a positive strength]}
is as follows: we mentioned earlier that $R(S_1, s)$ is
the strength of attack by $S_1$ on $s$, measured
in $\mathbb{N}$. The value being positive
signals $S_1$'s attack on $s$.
The value being 0 would mean that $S_1$ is not
attacking $s$. The purpose
of an attack relation in abstract argumentation
frameworks is to know which arguments
attack which arguments. It is for this reason
that we only
consider positive values for $R$. \textbf{[Attack
    monotonicity 1]} expresses the following reasonable
property: an attack may be occurring from
some $S_1 \subseteq S \subseteq_{\text{fin}} \mathcal{S}$
on some $s \in S$; now, increase the attack capacity
of just one argument $s_1 \in S_1$, keeping
all else equal; then the attack which occurred
before the capacity increase should still occur.
\textbf{[Attack monotonicity 2]} expresses the property
that if an attack occurs from a set of arguments
on an argument with some strength, then any superset
does not decrease the attack strength.
To explain
\textbf{[Attack monotonicity 3]},
let us say that
a set of arguments is attacking an argument
with certain argument capacity. This intuitively means
that the set intends to suppress the argument.
Now, if the argument capacity of the argument
increases, the set still intends to suppress
the argument just as strongly or even more strongly,
but not less strongly, for there are more materials
in the argument that the set could attack. This
is the direct reading of the condition.
In the technical development to follow, the converse reading
will be more useful: if the argument
capacity of the argument on the other hand decreases,
the set may no longer intend to suppress it further
(because an argument with no content ceases
to be any argument in an ordinary sense).
Finally, \textbf{[No self attacks]} prevents self-contradictory
arguments from being present.\\
\indent Additivity
of attack strengths is not postulated for $R$:
if an argument $s_1$ is attacked by $s_2$ and $s_3$
such that $R(\{s_2\}, s_1) = n_1$ and that
$R(\{s_3\}, s_1) = n_2$,
it is not necessary that
$R(\{s_2, s_3\}, s_1) = n_1 + n_2$. The additivity
    holds good when each attack
    can be assumed independent, but may not
     in other cases.
    A generalised
version of [Attack  monotonicity 1] holds good.
Let $\pi$ be a projection function that takes
a natural number and an ordered tuple $\Gamma$ (not
the particular symbol $\Gamma$ but any
ordered tuple)
and that outputs a set  member,
such that $\pi(n, \Gamma) = $ $\{$the $n$-th
    member of $\Gamma$$\}$. It is undefined
if $n$ is greater than the size of the ordered set.
\begin{proposition}[Generalised attack
    monotonicity 1]
    Let $S_1 \subseteq_{\text{fin}} \mathcal{S}$
    and $s \in \mathcal{S}$ be such that
    $R(S_1, s)$ is defined. Then
    if $S_2 \subseteq_{\text{fin}} \mathcal{S}$
    is such that: (1) $\bigcup_{s_x \in S_1} \pi(1, s_x)
    = \bigcup_{s_x \in S_2} \pi(1, s_x)$; and that
    (2) $(a, n) \in S_1$ materially implies
    $(a, m) \in S_2$ for $n \le m$,
    then $R(S_2, s)$ is defined.
\end{proposition}
\begin{proof}
    \normalfont
    By induction on the number of
    arguments $s_1 \in S_1$ and $s_2 \in S_2$
    for which $\pi(1, s_1) = \pi(2, s_2)$
    $\andC$ $\pi(2, s_1) < \pi(2, s_2)$.
Here and everywhere, we may make
use of $\andC$ (having the semantics of classical logic conjunction),
distinguishing
`and' in formal contexts from `and' in natural contexts
for greater clarity.
    The base case is vacuous.
    Use [Attack relation monotonicity]
    for inductive cases.
\end{proof}
Our argumentation framework
is $(S, R)$ for some coherent
set of arguments $S$ and for some $R$.
\subsection{Attacks}
We distinguish complete attacks (defeats) from
partial attacks (attacks).
\begin{definition}[Attacks and defeats]
    \normalfont
   We say that $S_1 \subseteq S$ attacks
   $s \in S$ iff there exists
   $S_2 \subseteq S_1$ such that
   $R$ is defined for $(S_2, s)$.
      We say that $S_1 \subseteq S$ defeats
   $s \in S$ iff
   $S_1$ attacks $s$ $\andC$  there exists
   some $S_2 \subseteq S_1$ such that:
   (1) $R$ is defined for $(S_2, s)$; and
   (2) if $R$ is defined for $(S_x, s)$ for $S_x \subseteq S_1$,
   then $\pi(2, s) \leq R(S_2, s)$.\footnote{See
       in the proof of Proposition 1 for
       what
       $\andC$ is.}
\end{definition}
Informally, an attack defeats its target
when the attack strength surpasses the target's
argument capacity.
\begin{definition}[Maximum attack strengths]
   \normalfont
   We define $\Rmax(S_1, s)$ to be:
   0 if $S_1$ does not attack $s$;
   otherwise, $R(S_2, s)$ for
   some $S_2 \subseteq S_1$ such that:
   (1) $R$ is defined for $(S_2, s)$;
   and (2)
   if $R$ is defined for
   $(S_x, s)$ for $S_x \subseteq S_1$, then
   $R(S_x, s) \leq R(S_2, s)$.
   \end{definition}
   \begin{example} \normalfont
       Let us build on Example 2. Notice that $(a_1, 4)$ disagrees
       with $(a_3, 5)$ on 3 points.
       We can model the attacks between them via
       $R(\{(a_1, 4)\}, \{(a_3, 5)\}) = R(\{(a_3, 5)\}, (a_1, 4)) = 3$,
       meaning that both of the attacks weaken the attacked argument
       by 3 sub-arguments.
   \end{example}
\subsection{Conflict-eliminable sets}
\begin{definition}[Conflict-eliminable sets of
    arguments]\normalfont
    We say that $S_1 \subseteq S$
    is conflict-eliminable
    iff there exists no $s \in S_1$
    such that $S_1$ defeats $s$.
\end{definition}
Conflict-eliminability is a weaker
notion of the usual conflict-freeness which is
the property that there exists no $s \in S_1$
such that $S_1$ attacks $s$.
For the rationale behind obtaining
this definition and using it as a primitive
entity in our argumentation framework, we
point back to Example 1 where infighting was evident and
where the LDP political policy did not contain
any internally conflicting arguments.
\begin{definition}[Intrinsic arguments]
    \normalfont
    Let $\alpha: 2^\mathcal{S} \rightharpoonup
    2^{\mathcal{S}}$ be
    such that it is defined
    for $S_1 \subseteq S$
    iff $S_1$ is conflict-eliminable.
    If $\alpha$ is defined for $S_1 \subseteq S$, then
    we define that
    $\alpha(S_1) =
   \{(\pi(1, s), n) \ | \
       s \in S_1\ \andC\
       n = \pi(2, s) - \Rmax(S_1, s)\}$.
   We say that $\alpha(S_1)$
   are intrinsic arguments
   of $S_1$.
\end{definition}
Intrinsic arguments of $\lkakko (a_1, 4), (a_3, 5) \rkakko$
in Example 2 are $\lkakko (a_1, 1), (a_3, 2) \rkakko$.
\begin{proposition}[Well-definedness]
    For any $S_1 \subseteq S$,
    if $\alpha$ is defined for $S_1$,
    then every member of $\alpha(S_1)$
    is a member of $\mathcal{S}$:
    in particular, there
    exists no $a \in \mathcal{A}$
    and no $n \in \mathbb{N}$
    such that $(a, -n) \in \alpha(S_1)$.
\end{proposition}
Intrinsic arguments of a conflict-eliminable
set must be substantively conflict-free.
\begin{definition}[The view of intrinsic arguments]
    \normalfont
Let $\Delete_R(S, S_x)$ be
$\{(S_y, s) \ | \ s \in S_x \subseteq S \ \andC\
    S_y \subseteq S_x
    \ \andC\ R(S_y, s)
    \text{ is defined.}
\}$, which is the set of attack relations within
$S_x$.  Now, let $S_1$ be a subset of $S$.
If $\alpha$ is defined for $S_1$, then
    we say that $((S \backslash S_1)
    \cup \alpha(S_1), R\backslash \Delete(S, S_1))$ is the view
    that $S_1$ has about $S$, or simply
     $S_1$'s view of $S$.
    We denote
    $S_1$'s view of $S$ by $\view(S, S_1)$.
\end{definition}
\begin{proposition}\label{prop_4}
   Let $S_1$ be a subset of $S$.
   If
   $\alpha(S_1)$ is defined,
   then $\alpha(S_1)$ is conflict-free
   in $\view(S, S_1)$.
\end{proposition}
\begin{example} \normalfont{\ }\\
  \includegraphics[scale=0.34]{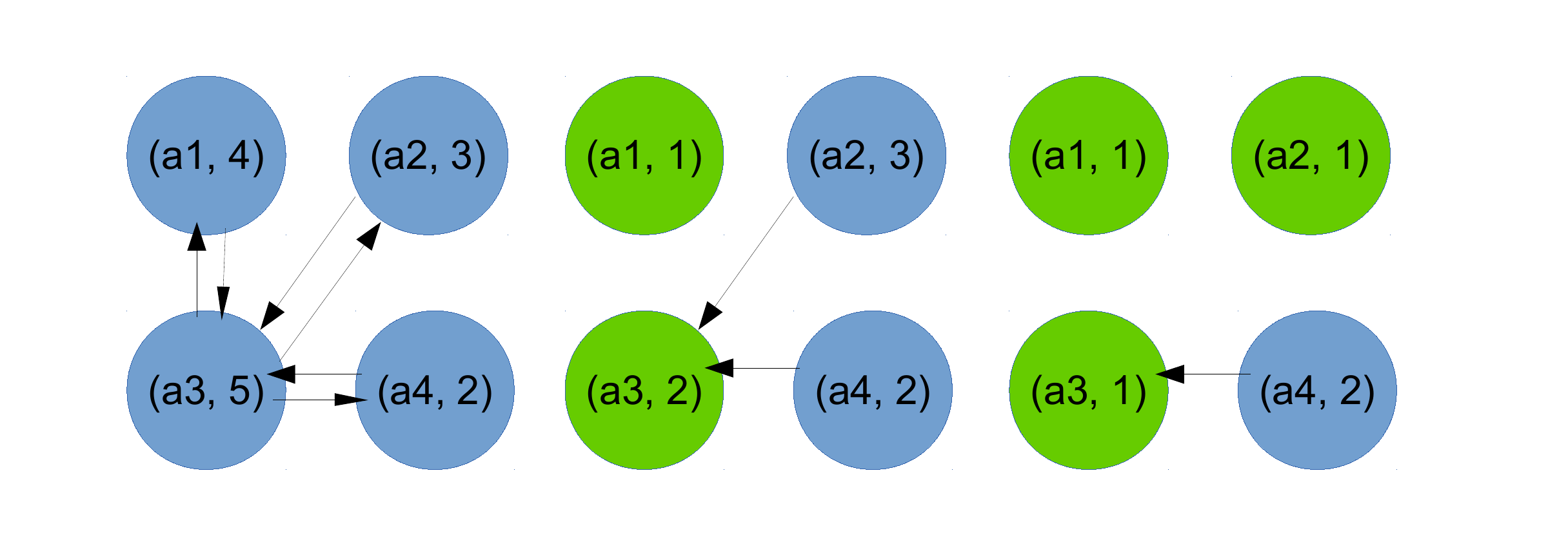}\\
The left drawing shows the arguments and attacks in our
running example when none of the arguments
are in a coalition with others.
The middle and the right drawings show the views that
$\lkakko (a_1, 4), (a_3, 5) \rkakko$
and $\lkakko (a_1, 4), (a_2, 3), (a_3, 5) \rkakko$
have about them.\footnote{
    These are not uniquely determined
    by the left drawing. We
    are deriving all the three drawings from
    the actual arguments found in
    Example 2 with the 8 conditions and Definition 6 as
    constraints.}
    Consider the middle figure.
  As we described earlier, three
  subarguments are in conflict
  between $(a_1, 4)$ and $(a_3, 5)$,
  and the intrinsic arguments of $\{(a_1, 4), (a_3, 5)\}$
  are $\{(a_1, 1), (a_3, 2)\}$.
$(a_1, 1)$
represents just one sub-argument
(X) for a Gasoline tax.
$(a_3, 2)$
represents just two sub-arguments (Y) for
military relationship with the U.S.
and (Z) for free trade for
car exports. By Definition 6 and
the monotonicity conditions of $R$,
there cannot be attacks between
$(a_1, 1)$ and $(a_3, 2)$ in the middle drawing.
There is a conflict
between (Z) and a sub-argument of
$(a_2, 3)$ on free trade, and so we
put attack
arrows between them in the middle drawing.
But none of (X), (Y) and (Z)
are in conflict with $(a_4, 1)$, which
explains why we do not put attack arrows between them
in the middle drawing.
Similarly for the right drawing, intrinsic
arguments are as shown in the figure.
Note that the subargument (Z) of $(a_3, 5)$
is also not in agreement with one subargument of $(a_2,3)$
in addition to those that are in conflict with
subarguments of $(a_1, 4)$. Consequently
the intrinsic arguments contain
$(a_3, 1)$ instead of $(a_3, 2)$. Meanwhile
for $(a_2, 3)$,
in addition to the policy on free trade,
the policy on international relation is
also in conflict with that of $(a_3, 5)$. Consequently,
the intrinsic arguments contain
$(a_2, 1)$.

\end{example}
\subsection{Coalition attacks, c-admissible sets and
c-preferred sets}
A coalition may attack external arguments
only by its intrinsic arguments.
\begin{definition}[C-attacks and c-defeats]
    \normalfont
   We say that
   $S_1 \subseteq S$ c-attacks $s \in S$ iff
   $\alpha$ is defined for $S_1$ $\andC$
   there exists some $S_2 \subseteq
   \alpha(S_1)$ such that
   $\pi(2, \view(S, S_1))$ is defined for
   $(S_2, s)$.
   We say that $S_1$
   c-defeats
    $s \in S$ iff $S_1$ c-attacks
    $s$ $\andC$
    $\pi(2, \view(S, S_1)))(S_3, s)
    \geq \pi(2, s)$ for some $S_3 \subseteq \alpha(S_1)$.
\end{definition}
    An example of c-defeat in our running example
    is by
    $\lkakko (a_3, 5) \rkakko$ on
    $(a_4, 1)$, and an example
    of c-attack is by
    $\lkakko (a_1, 4), (a_3, 5) \rkakko$
    on $(a_2, 3)$.  There are no self c-attacks; Cf. Proposition
\ref{prop_4}.
\begin{proposition}
    The following are equivalent.
    \begin{enumerate}
        \item $S_1 \subseteq S$
            c-attacks $s \in S$.
        \item $S_1 \subseteq S$
            c-attacks $s \in \pi(1, \view(S, S_1))$.
    \end{enumerate}
    Also, the following are equivalent.
    \begin{enumerate}
         \item $S_1 \subseteq S$
             c-defeats $s \in S$.
         \item $S_1 \subseteq S$
             c-defeats $s \in \pi(1, \view(S, S_1))$.
    \end{enumerate}
\end{proposition}
\begin{proof}
    \normalfont
    By definition,  $\pi(2, \view(S, S_1))(S_x, s)$ is
    undefined
    for any $S_x \subseteq S_1 \cup  \alpha(S_1)
    $
    and for any $s \in S_1 \cup \alpha(S_1)$.
    Hence if $S_1 \subseteq S$ c-attacks
    $s \in S$, then $s \not\in S_1$ $\andC$
    $s \not\in \alpha(S_1)$.
    Meanwhile, $(S \backslash S_1) =
    (\pi(1, \view(S, S_1)) \backslash \alpha(S_1))$.
    \end{proof}
We now define the notions of c-admissible and c-preferred sets,
which are analogous to admissible and preferred sets
we touched upon in Section 2, but which
are for conflict-eliminable sets.
One part in the definition could appear difficult
at first. We italicise  the part, and elaborate it later.
\begin{definition}[C-admissible/c-preferred sets]
    \normalfont {\ }\\
    We say that $S_1 \subseteq S$ is c-admissible
    iff $\alpha$ is defined for $S_1$ $\andC$
    {\it if $S_2 \subseteq \pi(1, \view(S, S_1))$
        attacks $s \in S_1$}
    and if, for any $S_x \subseteq S_2$ such that
    $R(S_x, s)$ is defined,
    there exists some $S_3 \subseteq \alpha(S_1)$
    such that $S_3$ c-defeat some $s_x \in S_x$.
    We say that
    $S_1 \subseteq S$ is c-preferred
    iff $S_1$ is c-admissible $\andC$
    there exists no $S_1 \subset S_y \subseteq S$
    such that $S_y$ is c-admissible.
\end{definition}
We explain why in the italicised part in the above
definition it is
$S_2 \subseteq \pi(1, \view(S, S_1))$ and not $S_2
\subseteq S$; and why it is $s \in S_1$ and not
$s \in \alpha(S_1)$.
The notion of c-admissibility
is intuitively the same as that of admissibility
in Section 2: subsets of $S_2$ are attacking
a conflict-eliminable
set $S_1$; and so $S_1$ is not admissible
unless some member of $S_1$ defeat all the attacking
subsets of $S_2$.
Now, because we are presuming that
$S_1$ is conflict-eliminable and not necessarily conflict-free,
$S_2 \subseteq S$ would include any partial
conflict in $S_1$ as an attack. However,
c-admissibility, which is the admissibility
of a conflict-eliminable set in the view of
the set,
should not be defined to defeat it. This explains
why $S_2 \subseteq \pi(1, \view(S, S_1))$ where
all those purely internal partial conflicts are compiled
away.
On the other hand, the reason that it should be
$s \in S_1$ and not
$s \in \alpha(S_1)$ is due to the asymmetry
in attacks
to and from a conflict-eliminable set. Recall \textbf{Asymmetry in attacks to and from a coalition} in Section 1. As a consequence,
$\{(a_1, 4), (a_3, 5)\}$ (and also
$\{(a_1, 4), (a_2, 3), (a_3, 5)\}$)
in Example 4 is not c-admissible because
the coalition $\{(a_1, 4), (a_3, 5)\}$ cannot
c-attack $(a_4, 1)$
while $(a_4, 1)$ attacks $(a_3, 5)$.
\subsection{Reduction to Nielsen-Parsons' argumentation
frameworks}
\begin{theorem}[Restricted $(S, R)$
    as Nielsen-Parsons']
    Let $R^{\downarrow}$ be such that
    it satisfies all but [Quasi-closure by subset relation] and [Closure by set union]. Let $S$ be such that
    for any $S_1 \subseteq S$ and for any $s \in S$,
    if
    $R^{\downarrow}$ is defined for $(S_2, s)$
    for some $S_2 \subseteq S_1$,
    then $S_1$ defeats $s$.
    Then $(S, R^{\downarrow})$ is Nielsen-Parsons'
    argumentation framework, and the following
    all
    hold good.
    \begin{enumerate}
        \item Any conflict-eliminable set
             in $S$ is a conflict-free set.
         \item If $\alpha$ is defined for $S_1
             \subseteq S$, then
             $\alpha(S_1) = S_1$ $\andC$
             $\pi(1, \view(S, S_1)) = S$.
         \item A c-attack by $S_1 \subseteq S$ on $s \in S$ is an attack by $S_1$ on $s$.
         \item A c-attack is a c-defeat.
         \item A c-admissible set is an admissible
             set, and a c-preferred set is a
             preferred set.
    \end{enumerate}
\end{theorem}
\begin{proof}  \normalfont
{\ }
\begin{enumerate}
    \item
     By definition, if $R^{\downarrow}$
    is defined for $(S_1, s)$, then $S_1$ defeats $s$.
   Hence
   it is necessary that a conflict-eliminable
   set be a conflict-free set.
   \item
    By 1., it is vacuous that $\alpha(S_1) = S_1$.
   $\Delete(S, S_1) = \emptyset$,
   and $\pi(1, \view(S, S_1)) = S$.
   \item
    By 2., both are trivial.
\item
    By definition of $(S, R^{\downarrow})$.
\item By 2., both are trivial. \end{enumerate}
With these, it is straightforward to see that
$(S, R^{\downarrow})$ is Nielsen-Parsons' argumentation
framework.
\end{proof}
The three monotonicity conditions may
be dropped, too. They are not relevant
to Nielsen-Parsons' frameworks.
 \section{Coalition profitability and formability Semantics}
It is of interest
to learn meaningful conflict-eliminable sets
of arguments. In this section we show semantic characterisations
of coalition formability out of conflict-eliminable
sets. Since coalition formation presupposes
at least two groups of arguments,
what we are characterising
is not whether a set of arguments
is admissible and how good
an admissible set over other admissible sets is, but
whether a conflict-eliminable
set can form a coalition with
another conflict-eliminable set, and
how good a coalition over other coalitions is.
We presume some conflict-eliminable set
at front, and will talk of coalition formability
semantics relative to the conflict-eliminable set.
We will first discuss profitability relation of a coalition,
will show its theoretical properties,
and will then present four coalition formability semantics,
each of which formalises certain utility postulate(s)
taking the profitability into account.
We assume the following notations.
\begin{definition}[One-directional attacks]
    \normalfont
   Let $S_1 \subseteq S$ be such that
   $\alpha(S_1)$ is defined.
    We say that $S_1$ is one-directionally
    attacked iff
    there exists
    $S_x \subseteq \pi(1, \view(S, S_1))$
             such that
             $S_x$ attacks $s \in S_1$
             $\andC$
             $S_1$ does not c-attack
             any $s_x \in S_x$.
\end{definition}
$\lkakko (a_1, 4), (a_2, 3), (a_3, 5) \rkakko$ in
the running example
in the previous section
is one-directionally attacked (by $(a_4, 1)$).
\begin{definition}[States of a conflict-eliminable
    set]
    \normalfont
    Let $\preceq: 2^\Sentence \times
    2^\Sentence$ be a binary relation
    such that
    $(S_1, S_2) \in \preceq$, written
    also $S_1 \preceq S_2$,
    iff $\alpha$ is defined both for $S_1$ and
    $S_2$ $\andC$ any of the three conditions
    below is satisfied:
    \begin{enumerate}
        \item $S_2$ is c-admissible.
        \item $S_1$ is one-directionally
            attacked.
        \item neither $S_1$ nor $S_2$
            is c-admissible
            or one-directionally attacked.
    \end{enumerate}
\end{definition}
Informally, if $S_x \subseteq S$ is c-admissible,
then it is fully defended from external attacks
and is good. If $S_x$ is one-directionally attacked,
then $S_x$ does not have
any answer to external attacks, which is bad.
Any conflict-eliminable set
that does not belong to either of them
is better than being one-directionally attacked
but is worse than being c-admissible.
Consequently, if $S_1 \preceq S_2$,
then $S_2$ is in a better state or in at least as good
a state as $S_1$.
\begin{definition}[Coalition permission]
    \normalfont
    We say that coalition
    is permitted between
    $S_1 \subseteq S$ and $S_2 \subseteq S$ iff
    $S_1 \cap S_2 = \emptyset$
    $\andC$ $\alpha$ is defined for $S_1 \cup S_2$.
\end{definition}
The following results tell that this definition of coalition permission
is not underspecified.
\begin{lemma}[Defeats are unresolvable]
    \label{mono}
    Let $S_1$ and $s$ be such that
    $S_1 \subseteq S$ $\andC$ $s \in S_1$.
    If $S_1$ defeats $s$,
    then for all $S_2$ such that
    $S_1 \subseteq S_2 \subseteq S$
    it holds that $S_2$ defeats $s$.
\end{lemma}
\begin{proof}  \normalfont
    By [Attack strength monotonicity] of $R$.
\end{proof}
\begin{proposition}[Coalition and conflict-eliminability]
    \label{conflict_eliminable}
    If coalition is permitted between $S_1$ and
    $S_2$, then $\alpha$
    is necessarily defined for $S_1$ and $S_2$.
\end{proposition}
\begin{proof}  \normalfont
    Suppose otherwise, then
    by conflict-eliminability
    of a set,
    there must be an argument
    in $S_i$, $i \in \{1,2\}$,
    such that $S_i$ defeats $s \in S_i$.
    Apply Lemma \ref{mono}.
\end{proof}
\indent We make one notion formally explicit for convenience,
and then define coalition profitability.
\begin{definition}[Attackers]
    \normalfont
    Let ${\Attacker: 2^{\Sentence} \rightarrow
        2^{\Sentence}}$ be such that\linebreak
 {\small $\Attacker(S_1) =
    \{s \in S \ | \
        \text{\small there exists
            some $s_1 \in S_1$
            such that $s$ }\linebreak
        \text{attacks $s_1$.}\}$}.
    We say that
    $\Attacker(S_1)$ is
    the set of attackers
    to $S_1$.
\end{definition}
\begin{definition}[Coalition profitability]
    \normalfont
    Let
$\unlhd: 2^{\Sentence} \times
2^{\Sentence}$ be such that if
$S_1 \unlhd S_2$ (or $(S_1, S_2) \in \unlhd$),
then three axioms below are all satisfied.
\begin{enumerate}
\item
    $S_1 \subseteq S_2$
         (larger set).\footnote{As illustrated
         in Introduction, our coalition scenario
         has this property.}
    \item
        $S_1 \preceq S_2$ (better state).
            \item
                $|\{s \in \Attacker(S_1) \ | \
                    \text{$S_1$ does not
                        c-defeat $s$ $\andC$
                    $s \not\in S_1$}\}|
                    \geq \\
                    |\{s \in \Attacker(S_1) \ | \
                        \text{$S_2$ does not
                            c-defeat $s$ $\andC$
                            $s \not\in S_2$
                        }\}|$
                    (fewer attackers).
            \end{enumerate}
\end{definition}
We say that $S_2$ is profitable for $S_1$ iff
$S_1 \unlhd S_2$.
Due to (better state), $S_1 \unlhd S_2$
implies that $\alpha$ is defined both for $S_1$ and $S_2$.
By (larger set), a set that contains more arguments
is a better set. By (better state),
a set that is in a better state is a better set.
Finally, by (fewer attackers),  a set that is attacked by
         a smaller number of arguments
         is a better set. The $S_1$
         in $\Attacker(S_1)$ on the second line
         is not a typo: this criterion is
         for measuring the profits
         of coalition formation
         for $S_1$.
    We have $\lkakko (a_1, 4), (a_3, 5) \rkakko
    \unlhd \lkakko (a_1, 4), (a_2, 3), (a_3, 5)
    \rkakko$  in Example 4.
\begin{proposition}
        A $(S, R)$ can be chosen in
    such a way that
    a pair of $S_1 \subseteq S$ and $S_2
    \subseteq S$
    do not
    satisfy more than one axioms of $\unlhd$.
\end{proposition}
\begin{proof}  \normalfont
    \begin{description}
        \item[(larger set)] {\ }\\
\includegraphics[scale=.32]{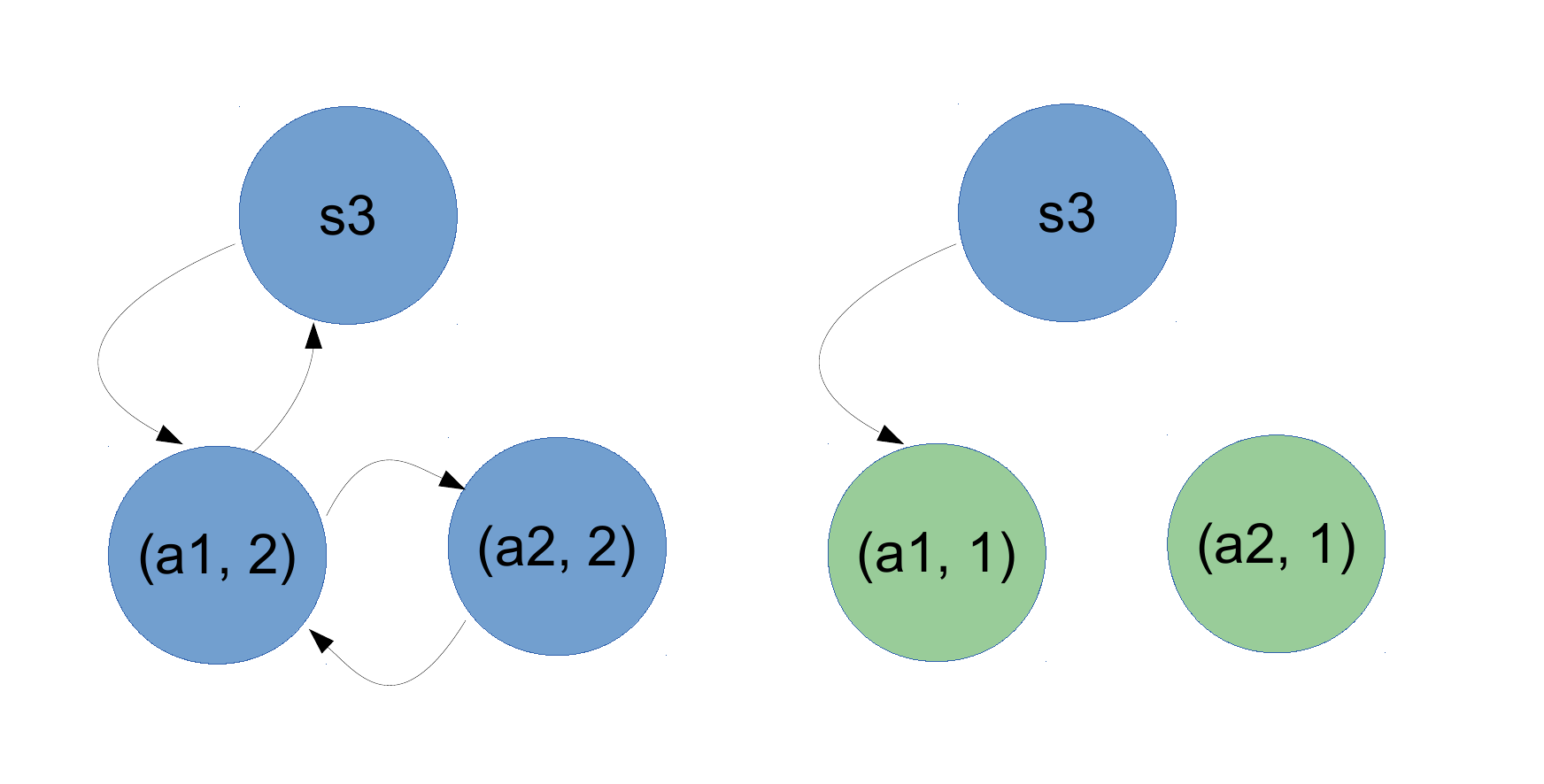} \\
             Let $S$ be
             $\{(a_1, 2), (a_2, 2), s_3\}$,
             and let $R$ be such that
             it is defined only for any combination that
             matches the attack arrows in
             the two drawings above.
                          Let us define that\linebreak
             $R(\{(a_1, 2)\}, (a_2, 2)) =
             R(\{(a_2, 2)\}, (a_1, 2)) =
             R(\{s_3\}, (a_1, 2)) =
             R(\{s_3\}, (a_1, 1)) =
             R(\{(a_2, 2), s_3\}, (a_1, 2)) = 1$,
             and that
             $R(\{(a_1, 2)\}, s_3) \geq \pi(2, s_3)$
             (among others implicit by the conditions of $R$).
             Then the left and the right drawings
             represent $(S, R)$ and respectively
             $\view(S, \{(a_1, 2),
                 (a_2, 2)\})$.
             Now, let $S_1$ be $\{(a_1, 2)\}$
             and let $S_2$ be
             $\{(a_1, 2), (a_2, 2)\}$.
                 Then
             clearly $S_1 \subseteq S_2$.
             However, it does not satisfy
             (better state):
             $S_1$ is neither c-admissible
             nor one-directionally attacked
             but $S_2$ is one-directionally attacked.
             It does not satisfy
             (fewer attackers), either: $\Attacker(S_1)
             =\{s_3\}$,
             $S_1$ c-defeats $s_3$, and
             $S_2$ c-defeats none.
         \item[(better state)]   {\ }\\
\includegraphics[scale=.32]{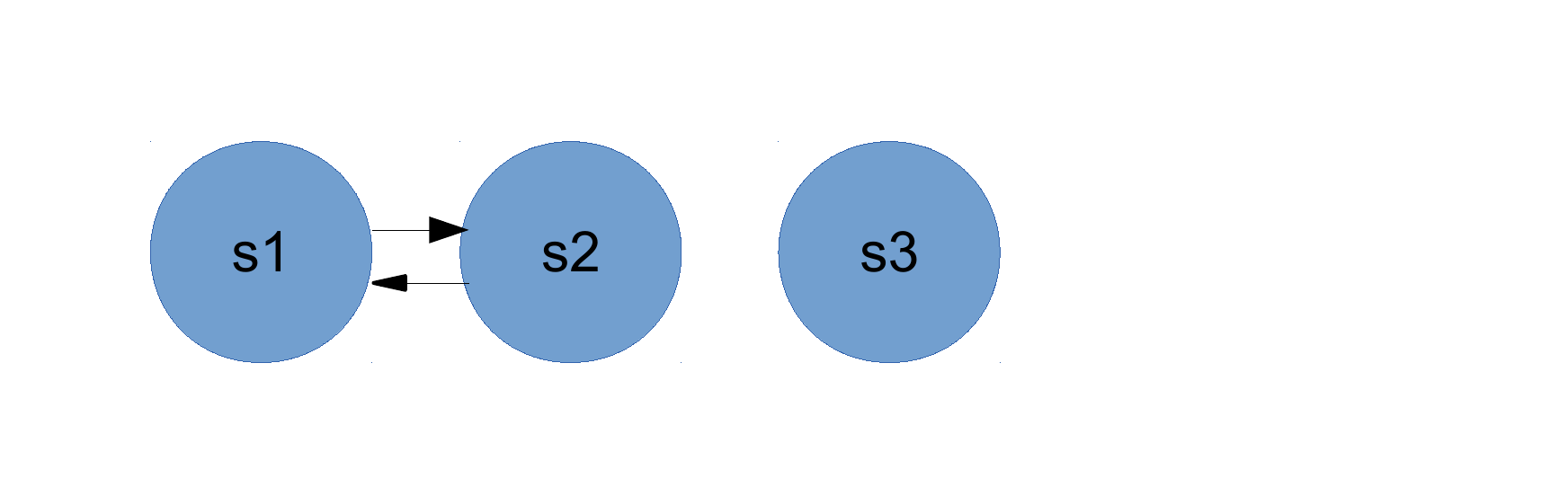} \\
             Let $S$ be $\{s_1, s_2, s_3\}$,
             and let $R$ be such that
             it is defined only mutually for
             $s_1$ and $s_2$ as shown above,
             and such that $R(\{s_1\}, s_2) \geq \pi(2, s_2)$
             $\andC$
             $R(\{s_2\}, s_1) \geq \pi(2, s_1)$.
             Now, let $S_1$ be $\{s_1\}$ and
             let $S_2$ be $\{s_3\}$.
             Then $S_1 \preceq S_2$ because
             $S_2$ is c-admissible. However,
             clearly it is not the case that
             $S_1 \subseteq S_2$.
             Also, $\Attacker(S_1) = \{s_2\}$,
             $S_1$ c-defeats $s_2$,
             but $S_2$ does not c-defeat it.
         \item[(fewer attackers)]
             Let $S$ be $\{s_1, s_2, s_3\}$,
             and let $R$ be such that
             it is defined as is previously, and such that
             $R(\{s_1\}, s_2) < \pi(2, s_2)$
             $\andC$
             $R(\{s_2\}, s_1) < \pi(2, s_1)$.
             Now,
             let $S_1$ be $\{s_3\}$,
             and let $S_2$ be $\{s_2\}$.
             Then $\Attacker(S_1) = \emptyset$,
             and so (fewer attackers) is satisfied
             trivially.
             However, clearly it is not the case
             that $S_1 \subseteq S_2$.
             And it is not the case that
             $S_1 \preceq S_2$, because
             $S_1$ is c-admissible,
             while $S_2$ is neither c-admissible
             nor one-directionally attacked.
    \end{description}
\end{proof}
\subsection{Theoretical results around profitability}
 It is easy to see that there exists some $(S, R)$
and some $S_1, S_2 \subseteq S$
such that $S_1 \unlhd S_1 \cup S_2$. We state
other results.
\begin{theorem}
    \label{thm_1}
    Let $S_1 \subseteq S$ be such that
    $\alpha(S_1)$ is defined,
    and let $S_x \subseteq S$ be a c-admissible set.
    If  $S_1 \subset S_x$, then
    the following all hold good.
    \begin{enumerate}
        \item $\alpha$ is defined for $S_2 = S_x \backslash S_1$.
        \item Coalition is permitted
            between $S_1$ and $S_2$.
        \item $S_1 \unlhd S_x$.
    \end{enumerate}
\end{theorem}
\begin{proof} \normalfont {\ }
    \begin{enumerate}
        \item Suppose otherwise,
            then $S_2$ would defeat at least
            one $s \in S_2$ if
            $S_2$  defeats $s$.
            By lemma \ref{mono},
            $S_x$ would then defeat $s$, for
            $S_2 \subseteq S_x$.
            But $S_x$, being  c-admissible,
            does not defeat $s$.
        \item $\alpha$ is defined for $S_x$ by definition; and $S_1 \cap S_2 = \emptyset$,
            also by definition.
        \item ${S_1 \subseteq S_x}$,
             $S_x$ is a c-admissible set,
             and\\
             $|\{{s \in \Attacker(S_1)} \ | \
                 \text{$S_x$ does not c-defeat
                     $s$ $\andC$ $s \not\in S_x$}\}| = 0$.
    \end{enumerate}
\end{proof}
 \begin{theorem}[Existence theorem]
    If, for any $S_1 \subseteq S$,
    there exists some $S_2 \subseteq S$
    such that coalition is permitted between
    $S_1$ and $S_2$ $\andC$
    $S_1 \unlhd S_1 \cup S_2$
    $\andC$ $S_1 \cup S_2$ is c-admissible,
     then there exists
     some $S_3 \subseteq S$ such that
     coalition is permitted between
     $S_1$ and $S_3$ $\andC$
     $S_1 \unlhd S_1 \cup S_3$
     $\andC$ $S_1 \cup S_3$ is c-preferred
     $\andC$ $S_2 \subseteq S_3
     $.
 \end{theorem}
 \begin{proof}  \normalfont
     As $S$ is of a finite size,
     it is straightforward to see
     that if $S_1 \cup S_2$
     is c-admissible, then
     there must exist some c-preferred set $S_x$
     such that $S_1 \cup S_2 \subseteq S_x$.
     We show that the choice of $S_x \backslash S_1$
     for $S_3$ ensures the requirement
     to be fulfilled. By Theorem \ref{thm_1}
     coalition is permitted between $S_1$ and $S_3$
     $\andC$ $S_1 \unlhd S_1 \cup S_3$.
     By assumption, $S_1 \cup S_3 = S_x$
     is c-preferred $\andC$ $S_2 \subseteq S_3$.
 \end{proof}

 \begin{theorem}[Mutually maximal coalition]
    \label{best_coalition}
   Let $S_1 \subseteq S$ be such that
   $\alpha(S_1)$ is defined, and let $\Pref(S_1)$ be the set
   of all c-preferred sets that contain
   $S_1$ as their subset. If $\Pref(S_1) \not= \emptyset$,
   then the following holds good:
   for any $S_x \in \Pref(S_1)$,
   $S_1 \unlhd S_x$ $\andC$
   $(S_x \backslash S_1) \unlhd S_x$ $\andC$
   there exists no $S_x \subset S_y \subseteq S$
   such that $S_1 \unlhd S_y$ or such that
   $(S_x \backslash S_1) \unlhd S_y$.
\end{theorem}
\begin{proof} \normalfont
    By Theorem \ref{thm_1} and
    by the definition of a set being c-preferable.
\end{proof}
In general, though,
the mutual profitability is not a guaranteed property.
\begin{theorem}[Asymmetry of profitabilities]
    \label{asymmetry}
     There exists an argumentation framework $(S, R)$
     with disjoint subsets:
     $S_1$ and $S_2$, of $S$ satisfying
     $S_1 \unlhd S_1 \cup S_2$ but
     not satisfying
      $S_2 \unlhd S_1 \cup S_2$.
 \end{theorem}
\begin{center}
     \includegraphics[scale=.32]{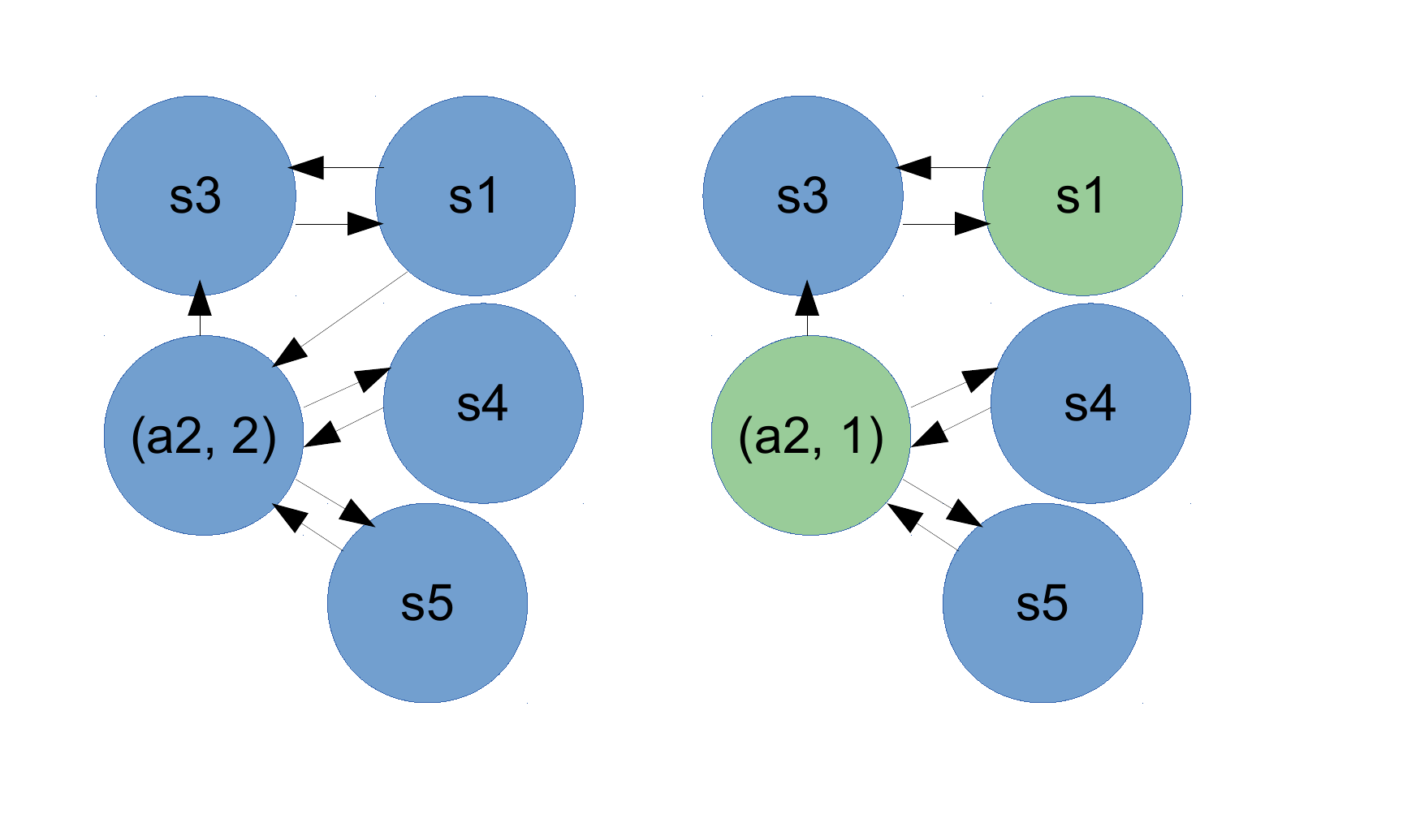}
\end{center} \begin{proof}  \normalfont
     Let $S$ be $\{s_1, (a_2, 2), s_3, s_4, s_5\}$,
     and let $R$ be defined
     only for any combination that
     matches the attack arrows in the drawings above,
     satisfying:
     ${R(\{s_1\}, (a_2, 2)) = 1}$; \\
     ${R(\{s_1, (a_2, 1)\}, s_3) \geq
         \pi(2, s_3)}$;
     ${R(\{s_1\}, s_3) < \pi(2, s_3)}$; \\
     ${R(\{(a_2, 2)\}, s_3) < \pi(2, s_3)}$;
     ${R(\{(a_2, 1)\}, s_4) < \pi(2, s_4)}$; \\
     ${R(\{(a_2, 1)\}, s_5) < \pi(2, s_5)}$;
     ${R(\{(a_2, 2)\}, s_4) \geq \pi(2, s_4)}$; \\
     ${R(\{(a_2, 2)\}, s_5) \geq \pi(2, s_5)}$;
     ${R(\{s_3\}, s_1) \geq 2}$; \\
     ${R(\{s_4\}, (a_2, 1)) \geq 2}$;
     and ${R(\{s_5\}, (a_2, 1)) \geq 2}$,
         among others. \\
         \indent Now, let $S_1$ be $\{s_1\}$,
     and let $S_2$ be $\{(a_2, 2)\}$.
     Then $S_1 \unlhd S_1 \cup S_2$, for
     $S_1$ and $S_1 \cup S_2$ are both
     neither c-admissible nor one-directionally
     attacked.
     The axiom (fewer attackers) is also satisfied.
     However, it is not the case that
     $S_2 \unlhd S_1 \cup S_2$,
     for it does not satisfy (fewer attackers).
 \end{proof}
Moreover, $\unlhd$ does not satisfy
what we may at first expect to hold good, i.e.
the following continuation property.
\begin{definition}[Continuation property
    of $\unlhd$]
    \normalfont
    Let $S_1 \subseteq S$ be such that
    $\alpha(S_1)$ is defined,
    and let $\Maxi(S_1)$ be the set of
    all $S_x \subseteq S$ such that
    $S_1 \unlhd S_x$ and such that
    if $S_x \subset S_y \subseteq S$,
    then not $S_1 \unlhd S_y$.
    We say that
    $\unlhd$ is weakly continuous for $S_1$
    iff there exists some $S_z \in \Maxi(S_1)$
    such that, for any $S_w \subseteq S_z$,
    if
    coalition is permitted between $S_1$ and $S_w \backslash
    S_1$,
    then $S_1 \unlhd S_w$.  We say that
    $\unlhd$ is continuous for $S_1$
    iff it is weakly continuous for $S_1$
    for any $S_z \in \Maxi(S_1)$.
\end{definition}
\begin{theorem}[Profitability discontinuation theorem]
    There exist $S_1, S_2, S_x \subseteq S$
    such that: (1) $\alpha$ is defined for $S_1, S_2$ and $S_x$; (2) $S_x \in \Maxi(S_1)$;
    and (3) $S_2 \subseteq S_x$,
    but such that $S_1 \unlhd S_2$ does not hold
    good. Moreover,
    $\Maxi(S_1) = \Pref(S_1)$ would not change
    this result.
\end{theorem}
\begin{proof} \normalfont
    It suffices to consider the case
    where $\Maxi(S_1) = \Pref(S_1)$. \\
\includegraphics[scale=.32]{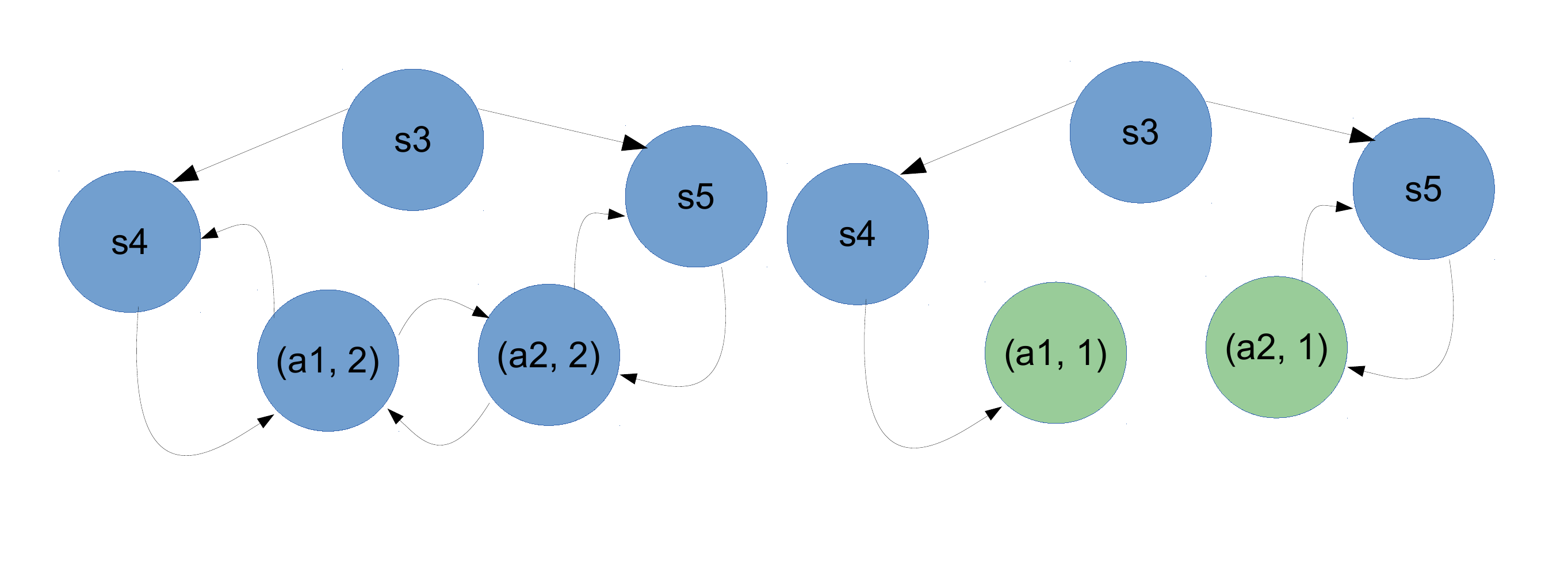} \\
    Let $S$ be $\{(a_1, 2), (a_2, 2), s_3, s_4, s_5\}$,
    and let $R$ be such that it is defined
    only for any combination that matches
    the attack arrows in the drawings above,
    and:
     ${R(\{(a_1, 2)\}, s_4) \geq \pi(2, s_4)}$;
    ${R(\{(a_2, 2)\}, s_5) \geq \pi(2, s_5)}$;
    ${R(\{(a_1, 2)\}, (a_2, 2)) = 1}$; \\
    ${R(\{(a_2, 2)\}, (a_1, 2)) = 1}$;
    ${R(\{s_3\}, s_4) \geq \pi(2, s_4)}$; \\
    ${R(\{(a_2, 1), s_3\}, s_5) \geq \pi(2, s_5)}$;
    ${R(\{s_4\}, (a_1, 1)) \geq 2}$;\\  and
 ${R(\{s_5\}, (a_2, 1)) \geq 2}$,
            among others. Clearly $\alpha$ is
            defined for $S_1 = \{(a_1, 2)\}$,
            $S_a = \{(a_2, 2)\}$
            $S_2 = S_1 \cup S_a$,
            and $S_x = \{(a_1, 2), (a_2, 2), s_3\}$,
            among others.
            Further, $S_x$ is c-preferrable.
           However, it is not the case that
           $S_1 \unlhd S_2$. In fact,
           it is also not the case that $S_a \unlhd S_2$,
     since $S_2$ is one-directionally attacked
     (see the right drawing for the attacks
     in $\view(S, \{(a_1, 2), (a_2, 2)\})$),
     whereas $S_1$ and $S_a$ are neither c-admissible
     nor one-directionally attacked
     (see the left drawing).
\end{proof}
Coalition profitability continuation property
holds good in certain special cases, however.
 \begin{theorem}[Profitability
     continuation theorem]
    Let $S_x \subseteq S$ be
    a c-preferred set.
    Then $\unlhd$ is weakly continuous for
    any $S_1 \subset S_x$
    iff any disjoint pair $S_y, S_z$ of subsets
    of $S_x$ satisfy $S_y \unlhd S_y \cup S_z$.
 \end{theorem}
 \begin{proof}    \normalfont
     \begin{description}
         \item[\textbf{If:}]
             Suppose, by way of showing
             contradiction, there are three
             disjoint subsets of $S_x$:
             $S_a, S_b$ and $S_c$, such that
             $S_a \unlhd {S_a \cup S_b} {\not\unlhd}
             {(S_a \cup S_b) \cup S_c}$.
              By assumption
              we have $(S_a \cup S_b) \unlhd
              (S_a \cup S_b) \cup S_c$, contradiction.
         \item[\textbf{Only if:}] By definition of
             weak continuation property of $\unlhd$.
     \end{description}
 \end{proof}
  \subsection{Coalition formability semantics}
  We use the profitability relation
  to express our coalition formability semantics.
  We set forth three rational utility postulates:
\begin{enumerate}[label=\Roman*]
       \item Coalition is good when it is profitable at least
           to one party.
       \item Coalition is good when
           it is profitable to both parties.
       \item Coalition is good when
           maximal potential future profits
           are expected from it.
   \end{enumerate}
Of these, the first two can be understood immediately
with the profitability relation. Say there
are two sets $S_1$ and $S_2$ between which
coalition is permitted. For I,
at least either $S_1 \unlhd S_1 \cup S_2$
or $S_2 \unlhd S_1 \cup S_2$ must hold good.
In comparison, both of them must hold good
for II.
Our interpretation of the last postulate is as follows.
Suppose a party, some conflict-eliminable set $S_1 \subseteq S$ in our context,
considers coalition formation with another conflict-eliminable
set $S_2$.
We know that $S_2$ is some subset of $S_1 \subseteq S \backslash S_1$. Before $S_1$ forms a coalition
with $S_2$, we have $\Maxi(S_1)$ as the set of
maximal coalitions possible for $S_1$. Once the coalition is formed,
we have $\Maxi(S_1 \cup S_2)$ as the set of
maximal coalitions possible for the coalition. Here clearly
$\Maxi(S_1 \cup S_2) \subseteq \Maxi(S_1)$. What this
means is that a particular choice of $S_2$ blocks
any possibilities in $\Maxi(S_1) \backslash \Maxi(S_1
\cup S_2)$:  they become unrealisable from $S_1 \cup S_2$.
Hence $S_1$ has an incentive not to form
a coalition with a $S_2$ if all the members
of $\Maxi(S_1 \cup S_2)$ are strictly and comparatively less
profitable than some member of $\Maxi(S_1)$.
We reflect
this intuition.
\begin{definition}[Maximal profitability relation]
    \normalfont       {\ }\\
%
%
Let $\leq_{l}, \leq_{b}, \leq_{f}:
    2^{\mathcal{S}} \times 2^{\mathcal{S}}$
    be such that they satisfy all the following:
    \begin{enumerate}
        \item $S_1 \leq_l S_2$ iff
            $|S_1| \leq |S_2|$.
        \item $S_1 \leq_b S_2$ iff
            $S_2$ is at least as good by (better state)
            as $S_1$.
        \item $S_1 \leq_f S_2$ iff
            $S_2$ is at least as good by (fewer attackers)
            as $S_1$.
    \end{enumerate}
    We write $S_1 <_{\beta} S_2$ for each
    $\beta \in \{l,b,f\}$ just when
    $S_1 \leq_{\beta} S_2$
    and
    not $S_2 \leq_{\beta} S_1$.
Then we define $\unlhd_{\text{m}}: 2^{\mathcal{S}} \times
    2^{\mathcal{S}}$ to be such that
    if $S_1 \unlhd_{\text{m}} S_2$, then
    both of the following conditions satisfy:
    \begin{enumerate}
        \item $S_1 \unlhd S_2$.
    \item  Some $S_x \in \Maxi(S_2)$ is such that,
        for all $S_y \in \Maxi(S_1)$,
        if $S_x <_{\beta} S_y$ for
        some
        $\beta \in \{l, b, f\}$, then
        there exists
        $\gamma \in (\{l,b,f\} \backslash \beta)$
        such that
        $S_y <_{\gamma} S_x$.
\end{enumerate}
\end{definition}
Intuitively,
if $S_1 \unlhd_{\text{m}} S_2$,
then at least one set in $\Maxi(S_1)$ maximal
in the three criteria: the set size, the state quality
and the number of external attackers, is reachable from $S_2$.
The maximality here is judged by the principle
that if $S_2, S_3 \in \Maxi(S_1)$ are equal
by two criteria, then $S_3$ is better if
it is better than $S_2$ by the remaining criterion. \\
\indent We define four formability semantics: \textsf{W} which respects
   I, \textsf{M} which respects II (and implicitly
   also I), \textsf{WS} which respects
   I and III,
   and \textsf{S} which respects II and III.
\begin{definition}[Coalition formability semantics]
    \normalfont
     \begin{align*}
       \textsf{W}(S_1) &=  \{S_2 \subseteq S
               \ | \ S_1 \unlhd S_1 \cup S_2
               \ \orC\ S_2 \unlhd S_1 \cup S_2\}.\\
        \textsf{M}(S_1) &=  \{S_2 \subseteq S
               \ | \ S_1 \unlhd S_1 \cup S_2
               \ \andC\ S_2 \unlhd S_1 \cup S_2\}.\\
           \textsf{WS}(S_1) &=
           \{S_2 \subseteq S \ | \
               S_1 \unlhd_{\text{m}} S_1 \cup S_2
               \ \orC \ S_2 \unlhd_{\text{m}} S_1 \cup
               S_2
           \}.\\
           \textsf{S}(S_1) &=
           \{S_2 \subseteq S
               \ | \ S_1 \unlhd_{\text{m}} S_1  \cup S_2
               \ \andC \ S_2 \unlhd_{\text{m}} S_1
               \cup S_2\}.
   \end{align*}
\end{definition}
Here,
         $\orC$ has the semantics of classical logic
         disjunction.
Intuitively,  $\rho(S_1)$ for
$\rho \in \{\textsf{W}, \textsf{M}, \textsf{WS},
    \textsf{S}\}$ means that
$S_1$ is comfortable with forming a coalition with
$S_2 \in \rho(S_1)$ under the given criteria.
\begin{theorem}
     \label{relation}
    Let $S_1 \subseteq S$ be such that
    $\alpha(S_1)$ is defined.
    The following all hold good.
    (1)
         $\textsf{M}(S_1) \subseteq \textsf{W}(S_1)$.
         (2)
        $\textsf{WS}(S_1) \subseteq \textsf{W}(S_1)$.
        (3) $\textsf{S}(S_1) \subseteq \textsf{M}(S_1)$.
        (4) $\textsf{S}(S_1) \subseteq \textsf{WS}(S_1)$.
        Meanwhile, neither
        $\textsf{WS}(S_1) \subseteq \textsf{M}(S_1)$
nor  $\textsf{M}(S_1) \subseteq \textsf{WS}(S_1)$
is necessary.
 \end{theorem}
We provide one example
to illustrate the semantic
differences.
  Let $S = \{s_1, (a_2, 2), (a_3, 2), s_4, s_5, s_6,
  s_7\}$,
 and let $R$ be such that it is defined
 only for any combination that matches
 the attack arrows in the above drawings,
 and such that:
 \begin{adjustwidth}{1cm}{}
 ${R(\{(a_2, 2)\}, (a_3, 2)) = 1}$;
 ${R(\{(a_2, 2)\}, s_4) \geq \pi(2, s_4)}$;\\
 ${R(\{(a_2, 2)\}, s_5) \geq \pi(2, s_5)}$;
 ${R(\{(a_2, 1), s_1\}, (a_3, 2)) \geq 2}$;\\
${R(\{(a_2, 1)\}, s_4) < \pi(2, s_4)}$;
 ${R(\{(a_2, 1)\}, s_5) < \pi(2, s_5)}$;\\
 ${R(\{s_1\}, (a_3, 2)) = 1}$;
 ${R(\{(a_3, 1)\}, s_1) \geq \pi(2, s_1)}$;\\
${R(\{(a_3, 1)\}, s_6) \geq \pi(2, s_6)}$;
${R(\{s_6\}, (a_3, 1)) = 2}$;\\
${R(\{s_6\}, s_7) \geq \pi(2, s_7)}$;
${R(\{s_7\}, s_6) \geq \pi(2, s_6)}$;\\
${R(\{s_7\}, s_5) \geq \pi(2, s_5)}$;
${R(\{s_7\}, s_4) \geq \pi(2, s_4)}$;\\
${R(\{s_4\}, s_7) < \pi(2, s_7)}$;
${R(\{s_4\}, (a_2, 1)) = 2}$;\\
${R(\{s_5\}, (a_2, 1)) = 2}$;
${R(\{s_5\}, s_7) < \pi(2, s_7)}$;\\
${R(\{s_4, s_5\}, s_7) \geq \pi(2, s_7)}$, among others
that are implicit by the conditions of $R$.
\end{adjustwidth}
\begin{center}
\includegraphics[scale=0.25]{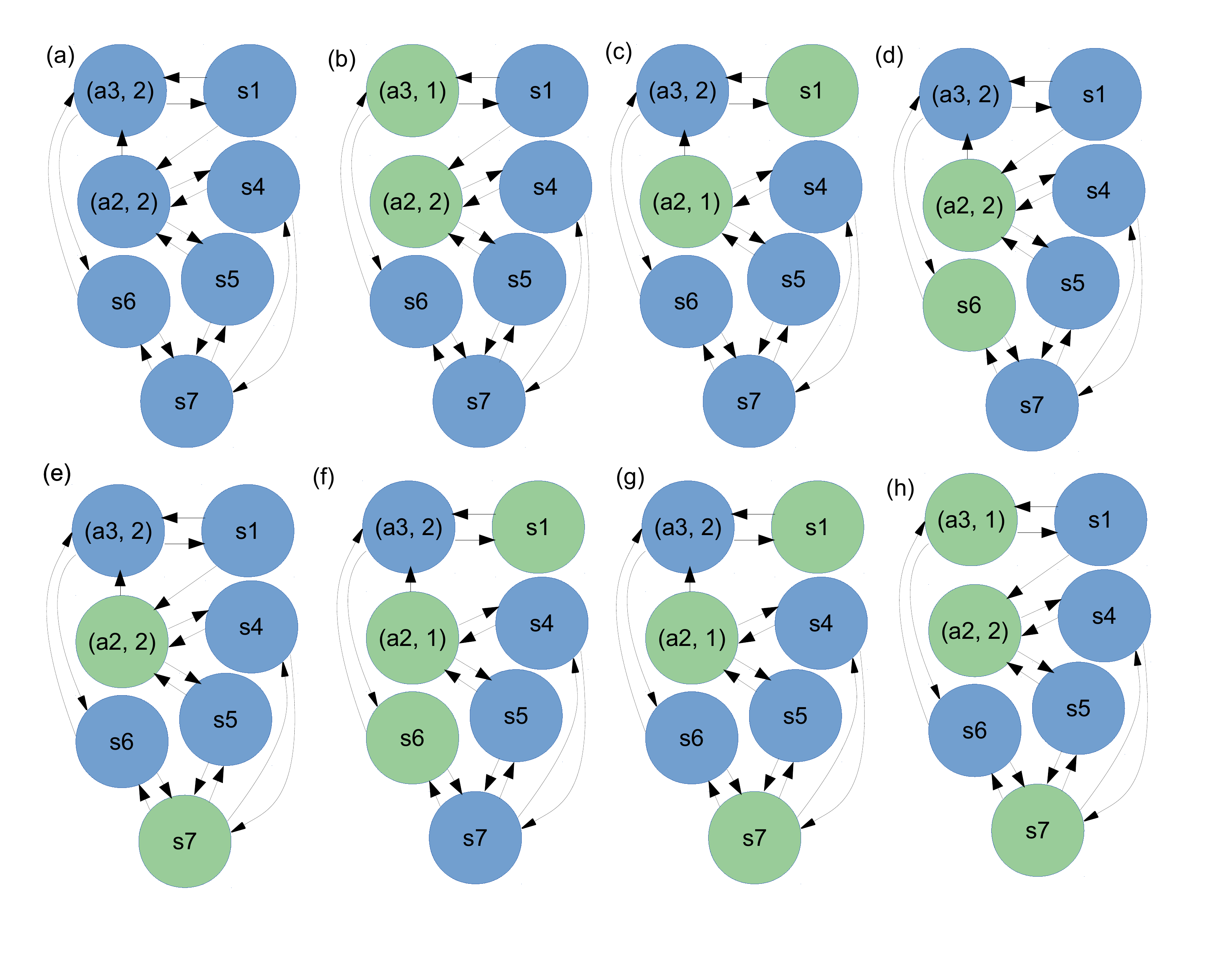}
\end{center}
We have:
{\small
 \begin{align*}
         \textsf{W}(\{(a_2, 2)\}) &=  \{\{(a_3, 2)\},
             \{s_1\}, \{s_6\},\{s_7\}, \{s_1, s_6\},
             \{s_1, s_7\},\\ &{\ }\quad\{(a_3, 2), s_7\}
         \}.\\
     \textsf{M}(\{(a_2, 2)\}) &= \{\{(a_3, 2)\},
         \{s_7\}, \{s_1, s_6\}, \{s_1, s_7\}, \{(a_3, 2), s_7\}\}.\\
     \textsf{WS}(\{(a_2, 2)\}) &= \{\{(a_3, 2)\}, \{s_1\}, \{s_7\}, \{s_1, s_7\},
         \{(a_3, 2), s_7\}\}.\\
     \textsf{S}(\{(a_2, 2)\}) &=
     \textsf{WS}(\{(a_2, 2)\}).
 \end{align*}
}
The above drawings from (b) to (h) correspond
to ${\view(S,\ \{(a_2, 2)\} \cup S_i)}$ where
$S_i$ is the $i$th element in the set
on the right hand side of the first equation.
The first two equalities can be checked one by one.
For {\small $\textsf{WS}(\{(a_2, 2)\})$},
we note that $\{s_1,  s_6\}$ (see (f)) is excluded because:
(1) for $\{(a_2, 2)\}$,
$\{(a_2, 2), (a_3, 2), s_7\}$ and
$\{s_1, (a_2, 2), s_7\}$ (see (g) and (h)) are both strictly
better in $\unlhd_{\text{m}}$ than
$\{s_1, (a_2, 2), s_6\}$ by (fewer attackers), and
$\{s_1, (a_2, 2), s_6\}$ would block
the better coalitions; and (2)
for $\{s_1, s_6\}$,
$\{s_1, s_6, s_4, s_5\}$ is strictly better
in $\unlhd_{\text{m}}$ than
$\{s_1, (a_2, 2), s_6\}$. Similarly $\{s_6\}$
is excluded.
Finally, in this particular example, both
$\{s_1, (a_2, 2),  s_7\}$ and
$\{(a_2, 2), (a_3, 2), s_7\}$ are c-preferred,
and, moreover, we have:
$\{(a_3, 2)\} \unlhd_{\text{m}}
\{(a_2, 2), (a_3, 2), s_7\}$;
$\{s_1\} \unlhd_{\text{m}} \{s_1, (a_2, 2), s_7\}$;
$\{s_7\} \unlhd_{\text{m}} \{s_1, (a_2, 2) s_7\}$;
and $\{s_7\} \unlhd_{\text{m}} \{(a_2, 2), (a_3, 2), s_7\}$.
By these together with the subsumption of $\textsf{S}(\{(a_2, 2)\})$
in $\textsf{WS}(\{(a_2, 2)\})$ (Theorem \ref{relation}),
the last equality for $\textsf{S}(\{(a_2,2)\})$ follows.
\section{Conclusion}
We proposed abstract-argumentation-theoretic
coalition profitability and formability semantics,
and showed theoretical results.
Our work
has a connection to several important subfields of
abstract argumentation theory such as
postulate-based abstract argumentation,
attack-tolerant abstract argumentation
and dynamic abstract argumentation.
It is our hope that
this study will further aid in
linking the rich knowledge that is
being accumulated in the literature.

 \bibliography{references}
\end{document}